\documentclass[sigconf]{aamas} 
\pdfoutput=1 
\usepackage{balance} % for balancing columns on the final page
\usepackage{times}  % DO NOT CHANGE THIS
\usepackage{helvet} % DO NOT CHANGE THIS
\usepackage{courier}  % DO NOT CHANGE THIS
\usepackage{natbib}  % DO NOT CHANGE THIS AND DO NOT ADD ANY OPTIONS TO IT
\usepackage{caption} % DO NOT CHANGE THIS AND DO NOT ADD ANY OPTIONS TO IT
\usepackage{graphics} % for pdf, bitmapped graphics files
\graphicspath{ {./figures/} }
\usepackage{amsmath} % assumes amsmath package installed
\usepackage{bm}
\usepackage{multicol}
\usepackage{booktabs}
\usepackage[linesnumbered,ruled]{algorithm2e}
\usepackage{subcaption}
\usepackage{mathtools}
\DeclareMathOperator*{\argmin}{arg\,min}
\DeclareMathOperator*{\argmax}{arg\,max}
\DeclareMathSizes{6.5}{6.5}{6.5}{6.5}

\usepackage[para]{footmisc}
\usepackage{cite}
\usepackage{istgame}
\usepackage{subcaption}
\usepackage{amsthm}
\usepackage{multirow}
\usepackage{array}
\usepackage{siunitx}
\usepackage{adjustbox}
\usepackage{bm}
\usepackage[clockwise]{rotating}
\usepackage{placeins}
\usepackage{physics}
\usepackage{csquotes}

\renewcommand{\mkbegdispquote}[2]{\itshape}
\usetikzlibrary{decorations.pathmorphing}
\newcolumntype{P}[1]{>{\centering\arraybackslash}p{#1}}
\newtheorem{defi}{Definition}
\newtheorem{theo}{Theorem}
\newtheorem{cor}{Corollary}
\newtheorem{proposition}{Proposition}

\newcommand*\ttvar[1]{\texttt{\expandafter\dottvar\detokenize{#1}\relax}}
\newcommand*\dottvar[1]{\ifx\relax#1\else
  \expandafter\ifx\string_#1\string_\allowbreak\else#1\fi
  \expandafter\dottvar\fi}
\setcopyright{ifaamas}
\acmConference[AAMAS '23]{Proc.\@ of the 22nd International Conference
on Autonomous Agents and Multiagent Systems (AAMAS 2023)}{May 29 -- June 2, 2023}
{London, United Kingdom}{A.~Ricci, W.~Yeoh, N.~Agmon, B.~An (eds.)}
\copyrightyear{2023}
\acmYear{2023}
\acmDOI{}
\acmPrice{}
\acmISBN{}

\acmSubmissionID{299}

\title[Revealed Multi-Objective Utility Aggregation]{Revealed Multi-Objective Utility Aggregation in Human Driving}

\author{Atrisha Sarkar}
\affiliation{
  \institution{University of Toronto}
  \city{Toronto}
  \country{Canada}}
\email{atrisha.sarkar@utoronto.ca}
\author{Kate Larson}
\affiliation{
  \institution{University of Waterloo}
  \city{Waterloo}
  \country{Canada}}
\email{kate.larson@uwaterloo.ca}
\author{Krzysztof Czarnecki}
\affiliation{
  \institution{University of Waterloo}
  \city{Waterloo}
  \country{Canada}}
\email{krzysztof.czarnecki@uwaterloo.ca}

\begin{abstract}
A central design problem in game theoretic analysis is the estimation of the players' utilities. In many real-world interactive situations of human decision making, including human driving, the utilities are multi-objective in nature; therefore, estimating the parameters of aggregation, i.e., mapping of multi-objective utilities to a scalar value, becomes an essential part of game construction. However, estimating this parameter from observational data introduces several challenges due to a host of unobservable factors, including the underlying modality of aggregation and the possibly boundedly rational behaviour model that generated the observation. Based on the concept of rationalisability, we develop algorithms for estimating multi-objective aggregation parameters for two common aggregation methods, weighted and satisficing aggregation, and for both strategic and non-strategic reasoning models. Based on three different datasets, we provide insights into how human drivers aggregate the utilities of safety and progress, as well as the situational dependence of the aggregation process. Additionally, we show that irrespective of the specific solution concept used for solving the games, a data-driven estimation of utility aggregation significantly improves the predictive accuracy of behaviour models with respect to observed human behaviour.
\end{abstract}


\newcommand{\BibTeX}{\rm B\kern-.05em{\sc i\kern-.025em b}\kern-.08em\TeX}

\begin{document}

%%% The following commands remove the headers in your paper. For final 
%%% papers, these will be inserted during the pagination process.

\pagestyle{fancy}
\fancyhead{}

%%% The next command prints the information defined in the preamble.

\maketitle 

%%%%%%%%%%%%%%%%%%%%%%%%%%%%%%%%%%%%%%%%%%%%%%%%%%%%%%%%%%%%%%%%%%%%%%%%

\label{chap:pref_est}
\section{Introduction}
In part due to interest in autonomous vehicle (AV) research, recent years have seen a rise in the use of game-theoretic models for modelling human driving behaviour. In this context, the estimation of the utilities of the agents is one of the main steps involved in the design of game-theoretic models based on observational data. However, estimating utilities of players, in this case, human drivers, from purely observational data raises several challenges. First, the problem involves multi-objective utilities; during driving, humans balance different potentially conflicting objectives, such as safety, progress, and comfort, in the process of selecting their desired action. The process of aggregation, defined as the mapping of multi-objective utilities into a scalar value \citep{Roy1971}, is often context dependent and individual specific. Second, the underlying human reasoning model is often diverse and cannot be modelled through a specific notion of an equilibrium behaviour in all situations \citep{sarkar2022generalized,sun2020game}. Additionally, there may also be aspects of bounded rationality in play \citep{schmidt2014prospect}. One way to estimate utilities from observational data is by using the concept of \emph{rationalisability}, that is, identifying the aggregation parameters that would make the observed agent decision optimal with respect to a reasoning model. However, this is still challenging since not only are there different modalities of aggregation, but also the definition of optimality depends on the various strategic and non-strategic reasoning models involved. When there are so many unobservable factors, the problem is inherently under-specified, meaning that the same observation can be explained by different combinations of reasoning model and utility aggregation.
In this paper, we address the problem of estimating the parameters of multi-objective utility aggregation from observational data by developing different methods for popular classes of utility aggregation, namely weighted and satisficing aggregation, and different underlying reasoning models that include both strategic and non-strategic reasoning. \par
%The utility $U$ was multi-objective, with safety and progress being the two dimensions, and these utilities were scaled using weighted aggregation using fixed weights. This process of fixing the aggregation parameter to a specific value is commonly encountered in the literature that uses multiobjective utilities \citep{wei2014behavioral,tian2018adaptive,medina2021merit}, especially since it helps focus the game construction and evaluation on models rather than introduce another free variable. However, addressing the question of how drivers aggregate safety and progress in their decision making is a critical question in the context of driving and requires further investigation. To this end, the goal of this paper is to develop methods that can estimate the aggregation parameters involved in $U$ that are rationalisable based on a given model of reasoning and empirical observation of naturalistic behaviour.\par
Analysing the aggregation process helps us gain a basic understanding of drivers' preferences under different driving situations, and answering questions such as observed association between the situational state and the utility aggregation. For example, if a driver is observed to have a higher than expected speed close to an intersection, can we infer something about their aggregation parameter based on that observation --- maybe that a driver in that context may weigh progress more than safety? From a modelling point of view, we obtain a more accurate identification of agent utilities, which is especially relevant for behavioural models, since in the absence of such an analysis, any behavioural model, no matter how incorrect the utilities are, can explain away deviations of observed behaviour under boundedly rational behaviour. The line of inquiry of estimating some aspect of agent utilities from observations is related to the problem of inverse reinforcement learning \citep{sadigh2016planning, kim2016socially}, inverse game theory \citep{kuleshov2015inverse}, theory of revealed preference \citep{crawford2014empirical}, and multi-criteria decision making in operations research \citep{doumpos2013multicriteria, jacquet2001preference}. In the related work section, we highlight the contribution of this paper in light of the extensive literature on these topics. \par
In the context of estimating utilities, it is important to note that there are two separate questions. First, the question of form, i.e., estimating the form and parameters of the utility function, say $u_{\text{safety}}(\delta):\mathbf{R}^{+} \rightarrow [-1,1]$ that maps the choices (e.g., the distance gap $\delta$) into a utility interval [-1,1]. This has been well studied within the literature of revealed preference with specific behavioural theories such as time discounting of utilities \citep{dziewulski2018revealed}, risk aversion \citep{echenique2015savage}, and prospect theory of loss aversion \citep{werner2019revealed}, some of which have been applied to driving \citep{schmidt2014prospect} and robotics \citep{kwon2020humans}. However, a second question around multi-objective utilities, i.e. how to estimate aggregation parameters of multi-objective utilities based on consistency of observations and reasoning model, has received comparatively less attention in the context of AV or human driving--- and that is the focus of this work. Specifically, we address the following questions.
\begin{itemize}
    \item \textbf{Aggregation}: Given a multi-objective utility $U$ and a parametric scalarization function $S(U,\theta)$, how do we estimate $\theta$ that is \emph{rationalisable} with a set of observed choices of all agents conditioned on a model of reasoning?
    \item \textbf{Bounded rationality}: How can the estimation of $\theta$ accommodate nonstrategic reasoning models such as \texttt{maxmax} or \texttt{maxmin}? 
    \item \textbf{State association}: Is there an observed association between state factors such as velocity, traffic situation, etc., and $\theta$? In other words, are the parameter values stable across different traffic situations?
    \item \textbf{Model performance}: How does the performance of different behaviour models change when utilities $U$ are constructed based on a learning-based technique that infers the aggregation parameter $\theta$ from the data?
\end{itemize}
The question of aggregation is addressed by constructing axiomatic conditions under which a set of observations is \emph{rationalisable} using a given parameterized aggregation method, namely weighted aggregation and satisficing aggregation. As a part of the second question, we show that such a construction is different for strategic and nonstrategic models, where the former can be formulated by a set of linear constraints and the latter as a set of nonlinear constraints. For the third question, we estimate the rationalisable parameters for different traffic situations and evaluate whether there are significant situational differences in how drivers aggregate the utilities, namely safety and progress. Finally, by treating the state factors as independent variables and the aggregation parameter as dependent, we use the data to learn a regression model (CART) that can predict the aggregation parameter in new situations, and use that method to evaluate the performance of the popular classes of behaviour models proposed in the literature for driving. 

\section{Related work}
This section spans three different fields of research that deal with similar problems in their own right. Namely, literature on the theory and applications of reveled preference from economics, multicriteria decision making from operations research, and inverse reinforcement learning from robotics and computer science.\par
\noindent \emph{Theory of revealed preferences: } The main problem addressed in this paper, i.e., estimation of agent preferences given a set of observations and a model, falls under the scope of the theory of revealed preferences. The literature on connection between revealed preference and utilities has built upon Afriats's approach \citep{afriat1967construction}, which defines axioms of existence of a utility function $u$ that can \emph{rationalise} a set of observed behaviours. Although most of the literature is focused on aggregate consumer demand problems \citep{varian2012revealed, demuynck2019samuelson}, revealed preference conditions can also be constructed for noncooperative strategic models such as Nash equilibrium \citep{cherchye2011revealed}; and Chambers et al. \citep{chambers2017general} lay the universality and existence conditions of such a construction for any model beyond just equilibrium. Covering the extensive literature on revealed preference in economics is outside of the focus and scope of this work; therefore, we refer to \citep{demuynck2019samuelson} as a good reference for that general literature. \par
Most economics models are based on rational choices, and given that this work builds models that include non-strategic behaviour (boundedly rational agents), it is relevant to include literature on revealed preference that is based on behavioural economics. Crawford \citep{crawford2014empirical} presents a review of the literature on revealed preference that covers behavioural theories and links to empirical evidence. Dziewulski \citep{dziewulski2018revealed} constructs the revealed preference conditions based on a model in which a single agent uses time-based discounting of their utilities with various discounting models such as quasi-hyperbolic and exponential. In contrast, this paper uses a model that is simpler in some way (one-shot game as opposed to dynamic game) and more complex in other way (multi-agent behaviour). Application of the construction from \citep{dziewulski2018revealed} for the case of driving in dynamic semi-cooperative setting is an interesting future direction of research, especially since discounted utilities are standard in reinforcement learning (RL) based methods, and RL has received a lot of attention from the AV community in recent years. Another behavioural attribute, altruism, has been consistently observed in an empirical setting, especially in the context of dictator games \citep{ben2004reciprocity}. Andreoni and Miller \citep{andreoni2002giving} set the construction of the revealed preference with respect to altruistic behaviour in a dictator game and find that only a quarter of the participants were selfish money maximisers and the rest passed the test of altruistic behaviour. More recently, Porter and Adams \citep{porter2016love} study revealed preference with respect to altruistic behaviour in the context of intergenerational wealth transfer, that is, transfer of money from an adult child to ageing parents. The study in \citep{porter2016love} varies different models of utility, from pure selfish behaviour to pure altruism, and finds that although more than 90\% of the participants pass the test of revealed preference (i.e., behaviour consistent with the models and utilities), there were differences observed based on whom they were playing the game against, whether parents or strangers. Similarly to \citep{porter2016love}, we vary the utility construction (different models of aggregation), construct the revealed preference conditions, and test on empirical data (\citep{porter2016love} is based on a laboratory experiment) to evaluate what proportion of behaviour passes those conditions. However, the models and applications in this paper are, of course, quite different. Overall, although the above works have treated different behavioural attributes with respect to theory of revealed preference well, to our understanding there is no existing work on revealed preference that is based on multi-objective utilities and non-strategic reasoning models in the context of driving behaviour.

\noindent \emph{Multi-criteria decision analysis (MCDA)} Another strand of literature that is related to this paper is on multi-criteria decision making from operations research \citep{doumpos2013multicriteria}. Compared to the theory of revealed preference, where the focus is more on the model of decision making, in MCDA, the focus is on multiobjective nature of the utilities. The process of estimating the parameters of the aggregation process that an agent uses is called preference disaggregation (a terminology we retain in the paper), and Jacquet-Lagreze and Siskos \citep{jacquet2001preference} provide a review of the tools and techniques for that purpose until the year 2000. From a set of datapoints of ranked choices made by an individual, typical algorithms solve the general minimisation problem $\argmin\limits_{\mathbf{w}} ||\mathcal{R}(X)^{o},\mathcal{R}(X,\mathcal{A}_{\mathbf{w}})||$, where $\mathcal{R}(X)^{o}$ is the observed ranking of the alternatives by the agent, $\mathcal{R}(X,\mathcal{A}_{\mathbf{w}})$ is the ranking based on the aggregation model $\mathcal{A}$ parameterized on $\mathbf{w}$ on the same set of alternatives $S$. Standard algorithms, such as UTA \citep{jacquet1982assessing} formulate the solution as a mathematical programming problem, and in recent years, statistical learning methods similar to those we use (CART) have also been used \citep{doumpos2013multicriteria}. There are few differences between the MCDA methods and those in this paper. First, in our case, the models we study are strategic (and non-strategic) decisions, thereby adding another layer of complexity. Second, our problem in this paper is also less well-defined, since we do not have access to the drivers' ranking of the preferences, but rather only a singular choice of the observed action.

\noindent \emph{Inverse reinforcement learning (IRL):}
Although IRL \citep{ng2000algorithms} is conceptually different from the methods presented in this paper, it is relevant to include some recent works in the literature due to the interest and application of IRL for autonomous driving. IRL formulates the problem of estimating agents' behaviour as a single agent problem as opposed to the game theoretic approach of treating the problem as one of multi-agent behaviour with support for different reasoning processes. Another salient distinction in IRL is that it typically retrieves the utility $U$ that fits the observed behaviour best without referencing utility to prespecified dimensions of safety, progress, comfort, etc., but rather uses a single objective function that may or may not have a semantic meaning. Sadigh et al. \citep{sadigh2016planning} use IRL to first learn a policy of behaviour from demonstrations and subsequently use that in a game theoretic based planning module using a level-k (k = 2) \citep{costa2001cognition} type solution concept, although the solution concept is not explicitly stated as such in the paper. As a mathematical formulation, such an approach works well in practise because IRL can provide a best-response type behaviour to the (other) agent action; however, the implicit assumption that the agents adhere to a single model of reasoning throughout every interaction might be a strong one. Nevertheless, the authors show practical ways to integrate a single agent method such as IRL into a game-theoretic setting.

A recent work on learning preference along multiple criteria with a game theoretic view and also in the context of driving is by Bhatia et al. \citep{bhatia2020preference}, where agent utilities are learnt with respect to a solution concept developed based on the Blackwell approachability theorem \citep{blackwell1956analog}. Compared to \citep{bhatia2020preference}, in this work, we use non-zero sum games and pure strategies in terms of the game constructs, as well as focus on multiple solution concepts. Additionally, in this paper, we also learn the preferences of drivers based on real-world observational data.  
\section{Aggregation}
\label{sec:aggregation}
The general problem of aggregation for an agent is the transformation of a vector valued utility function $U_{i}$ to a scalar valued function $u_{i}$ in order to solve the game in question. In other words, this involves the construction of a scalarization function $ \mathcal{S}(U_{i}(a_{i},a_{-i}),\theta_{i})$ that maps the multiobjective vector of utilities for agent $i$, $U_{i}(a_{i},a_{-i})$, to the real value utility $u_{i}(a_{i},a_{-i})$ based on the parameter $\theta_{i}$.
\subsection{Weighted Aggregation}
Weighted aggregation is a linear combination of individual utility objectives as follows.
\begin{equation}
    \mathcal{S}(U_{i}(a_{i},a_{-i}),\mathbf{w}_{i}) = \mathbf{w}_{i} \cdot U_{i}(a_{i},a_{-i})
\end{equation}
The above equation is simply the dot product between the aggregation parameter ($\mathbf{w}_{i}$ in this case) and the vector valued utility function. The disaggreagation process involves the estimation of the weight vector $\mathbf{w}_{i}$ based on the observed actions of the agents in the game.

\subsection{Satisficing Aggregation}
A driver always operating at their own subjective tolerance level of risk has been a well established model of behaviour in traffic psychology \citep{wilde1982theory, fuller1984conceptualization}, and the model has also been empirically validated \citep{lewis2010s}. Lexicographic thresholding is a method of aggregation that is based on \emph{satisficing} and encapsulates two concepts, namely, ordered criteria of objectives and a thresholding effect \citep{LiChangjian19}. In lexicographic thresholding, an agent ranks the objective criteria based on a fixed and strict total order, for example, safety $>$ progress $>$  comfort. In this work, we focus on safety and progress with a lexicographic ordering of safety $>$ progress. The aggregation of the two utilities into a scalar value is given by
\begin{equation}
    \mathcal{S}(U_{i}(a_{i},a_{-i}),\gamma_{i})=
\begin{cases}
  u_{s,i}(a_{i},a_{-i}), & \text{if }
       \begin{aligned}[t]
       u_{s,i}(a_{i},a_{-i})&\leq \gamma_{i}
       \end{aligned}
\\
  u_{p,i}(a_{i},a_{-}), & \text{otherwise}
\end{cases}
\label{eqn:satis_form}
\end{equation}
\noindent where $\gamma_{i}$ is the \emph{safety aspiration level} of agent $i$; $u_{s,i}(a_{i},a_{i})$ is the safety component of the vector valued function $U_{i}$; and $u_{p,i}(a_{i},a_{i})$ is the progress component. Based on the above formulation, an agent evaluates an action of multivalued utility based on progress rather than safety only when the safety utility of that action is greater than $\gamma_{i}$. The disaggregation process for the lexicographic thresholding method involves estimating the parameter $\gamma_{i}$ based on the observed action of the agent $i$ in the game.

\section{Multiobjective disaggregation}
We address the problem of estimating the aggregation parameters for an agent given their observed action in a game. This involves estimating the weight parameters $\mathbf{w}$ for the case of the weighted aggregation method, and the safety aspiration level parameter $\gamma$ for the case of the satisficing aggregation method. Additionally, since the choice of reasoning model (strategic or non-strategic) influences the behaviour of the agent in a game, we will develop separate methods based on strategic and non-strategic reasoning assumptions.\par
We start with the following definition of what \emph{rationalisability} means in the context of disaggregation of multiobjective utilities.
\begin{defi}
\label{chrv:defi:rationalisable_basic}
Given a normal form game $\mathcal{G}$, a vector-valued utility $U_{i}$, a solution concept $\mathcal{B}$, and a tuple of observed action $(a_{i}^{o},a_{-i}^{o})$, an aggregation parameter $\theta_{i}$ is rationalisable iff $(a_{i}^{o},a_{-i}^{o})$ is in the solution set of $\mathcal{G}$ solved with the solution concept $\mathcal{B}$ with scalarized utility $\mathcal{S}(U_{i},\theta_{i})$.
\end{defi}
Based on a dataset of observations $(a_{i}^{o},a_{-i}^{o})$ in various game situations, the goal is to estimate the rationalisable $\theta_{i}$ for each agent in each game situation. We first do this for the case where $\mathcal{S}$ is the weighted aggregation function, followed by the case where it is the satisficing aggregation function.
\begin{figure}[t]
\centering
\includegraphics[width=.75\columnwidth]{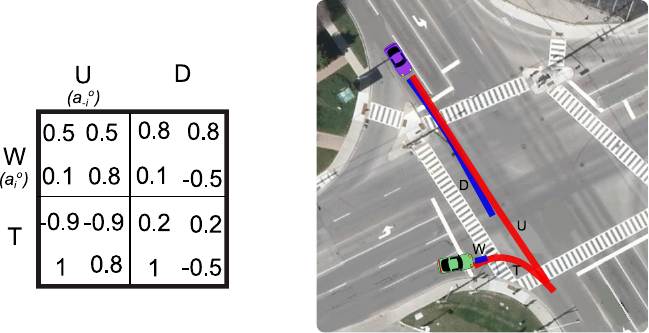}
  \caption{An example of right turning scenario with actions \textit{Turn} (T) and \textit{Wait} (W) with oncoming vehicle \textit{Speed up} (U) or \textit{Slow Down} (D). The first row in each cell is safety utility and second row is progress utility.}
 \label{fig:weighted_strategic_example}
\end{figure}
\subsection{Weighted Aggregation}
Since the observed actions depend on the underlying reasoning models used by the player, we first present the methods for strategic reasoning models (taking Nash equilibrium as an example), followed by non-strategic reasoning models of \texttt{maxmax} and \texttt{maxmin} \citep{wright2020formal}.
\subsubsection{Strategic models}
For strategic models, for the observed action, $a_{i}^{o}$, of a strategic agent $i$ to be in the solution set, the action needs to be the best response to the action that $i$ believes $-i$ will play. We use the case of Nash equilibrium in this section, where $a_{i}^{o}$ and $a_{-i}^{o}$ are best responses to each other. A running example of a right-turning scenario (Fig. \ref{fig:weighted_strategic_example}) elaborates the estimation process. A prototypical game for the scenario is shown on the right. For each combination of pure strategies, the top row utility values in each cell represent the safety utility and the bottom row represent the progress utility of the right turning vehicle (row player) and straight through (column player) vehicle, respectively. Let's say the observed strategy in this game was ($W,U$). In that case, for the right turning vehicle, for $W$ to be the best response to the observed action of $-i$, (i.e., $U$), the necessary and sufficiency conditions are $w_{i,s}\times 0.5 + w_{i,p} \times 0.1 \geq w_{i,s}\times -0.9 + w_{i,p} \times 1$ and $w_{i,s}+w_{i,p}=1$, where $w_{i,s}$ and $w_{i,p}$ are agent $i$'s weights for safety and progress utilities respectively. In the general case of arbitrary number of finite discrete actions and $|O|$ being the number of objectives, this can be formulated as a linear program (LP) and the rationalisable weights of agent $i$ can be estimated as the solution to the following LP
\begin{equation*}
\begin{array}{l@{}ll}
\text{maxmize}  & \displaystyle\sum\limits_{j=1}^{|O|} w_{i,j}u_{i,j}(a_{i}^{o},a_{-i}^{o}) &\\
\text{subject to}& \displaystyle\sum\limits_{j=1}^{|O|} w_{i,j}(u_{i,j}(a_{i}^{o},a_{-i}^{o})-u_{i,j}(a_{i}^{'},a_{-i}^{o})) \geq 0,  &\forall a_{i}^{'} \neq a_{i}^{o}\\
                 &     \displaystyle\sum\limits_{j=1}^{|O|} w_{i,j} = 1                                          &
\end{array}
\end{equation*}
In the above LP, we select weights in the feasible set that maximise the utility of the chosen action; however, any combination of weights that fall into the feasible set based on the constraints would be consistent with the conditions of rationalisability.
\subsubsection{Non-strategic models}
In the case of a non-strategic model, an agent is not \emph{other responsive} but only \emph{dominant responsive} \citep{wright2020formal}. In other words, since they do not reason about the actions of the other agents, it is not possible to pin down a specific action of the other agent ($a_{-i}^{o}$) with respect to which agent $i$ calculates its best response. In this model, the players do not best respond based on a specific belief about other agents' actions, but rather evaluate actions based only on their own utility values, and in our case, choose an action based on an elementary \texttt{maxmax} or \texttt{maxmin} model. This makes the process of estimating the weights slightly more complicated (read nonlinear) compared to the strategic case. Following from the example of Fig. \ref{fig:weighted_strategic_example}, for action $W$ to be the optimal action for agent $i$ (based on the non-strategic model \texttt{maxmax}), the maximum utility for agent $i$ that can be realised by choosing $W$ in the aggregate form post scalarization should be greater or equal to the maximum utility that can be realized by choosing $T$. Therefore, the necessary and sufficiency conditions for the non-strategic case in this example are $\text{max}\{w_{i,s}\times 0.5+w_{i,p}\times 0.1, w_{i,s}\times 0.8+w_{i,p}\times 0.1\} \geq \text{max}\{w_{i,s}\times -0.9+w_{i,p}\times 1, w_{i,s}\times 0.2+w_{i,p}\times -0.5\}$ and $w_{i,s}+w_{i,p}=1$. The left term in the inequality gives the maximum realised utility for the action $W$ and the right term is the maximum realised utility for the action $T$. The process of estimating the weights in the non-strategic case can therefore be formulated as a nonlinear optimisation problem as follows.
\begin{equation*}
\begin{array}{l@{}ll}
\text{maxmize}  & \displaystyle\sum\limits_{j=1}^{|O|} w_{i,j}u_{i,j}(a_{i}^{o},\argmax\limits_{a_{-i}}w_{i,j}u_{i,j}(a_{i}^{o},a_{-i})) &\\
\text{subject to}& \displaystyle\sum\limits_{j=1}^{|O|} w_{i,j}(u_{i,j}(a_{i}^{o},\argmax\limits_{a_{-i}}w_{i,j}u_{i,j}(a_{i}^{o},a_{-i}))&\\&-u_{i,j}(a_{i}^{'},\argmax\limits_{a_{-i}}w_{i,j}u_{i,j}(a_{i}^{'},a_{-i}))) \geq 0,  &\forall a_{i}^{'} \neq a_{i}^{o}\\
                 &     \displaystyle\sum\limits_{j=1}^{|O|} w_{i,j} = 1                                          &
\end{array}
\end{equation*}
Due to the presence of the \texttt{argmax} operator, the above problem changes to a nonlinear optimisation problem. Similarly, for \texttt{maxmin} non-strategic models, the process of estimating the weights is identical except that the \texttt{argmax} operator is replaced by the \texttt{argmin} operator. In the latter case, the \texttt{argmin} operator gives the minimum realisable utility of the observed action. In our experiments, for both \texttt{maxmax} and \texttt{maxmin} models, we solve the above optimisation problem using a trust region based method \citep{curtis2010interior}.
\subsection{Satisficing Disaggregation}
The estimation process for the satisficing method involves estimating the parameter $\gamma_{i}$ based on the observed action of the agent $i$ in the game. Similar to the weighted aggregation case, the method of estimation depends on the underlying model due to the assumption an agent has over other agents' behaviour and the subsequent impact on the optimality calculations based on that agent's perspective. However, unlike in the weighted aggregation case, due to the thresholding effect, it is not straightforward to construct a functional form of the scalar utility over which the optimality conditions can be built. Instead, we develop an algorithmic estimation process that helps estimate the complete set of rationalisable aggregation parameter $\gamma$. 
\subsubsection{Strategic model}
\begin{algorithm}[htp]
\caption{Estimation of $\Gamma_{i}$ based on consistency with respect to satisficing aggregation}
\label{algo:gamma_est_strategic}
  %\setcounter{AlgoLine}{0}
  %\SetKwFunction{algo}{estimate\_gamma\_strategic}\SetKwFunction{proc}{consistent\_set}

   \KwResult{$\Gamma{i}$}
   \SetKwInOut{Input}{Input}
   \Input{$(a_{i}^{o},a_{-i}^{o})$}
   $P \gets$\text{partition}($[-1,1],<$)\;
   $\Gamma_{i} \gets \{\varnothing\}$ \;
   \For{$I \in  P$}{
     $\gamma \gets$\text{sample}($I$)\;
     \If {\text{is\_rationalisable}($\gamma$)}{
        $\Gamma_{i} \gets \Gamma_{i} \cup I$
     }
    }

\end{algorithm} 
\sloppy Based on Eqn. \ref{eqn:satis_form}, the aggregation process for lexicographic thresholding can be expressed in a parametric form as $u_{i}(a_{i},a_{-i}) = \mathcal{S}(U_{i}(a_{i},a_{-i}),\gamma_{i})$ where $U_{i}(a_{i},a_{-i}) = [u_{i,s}(a_{i},a_{-i}),u_{i,p}(a_{i},a_{-i})]$ and $\mathcal{S}$ is the scalarization function of Eqn. \ref{eqn:satis_form}. We present an adapted definition of rationalisability for strategic models as follows:
\begin{defi}
\label{chrv:def:1}
For any agent $i$, a safety aspiration level $\gamma_{i} \in [-1,1]$ is equilibrium rationalisable with strategy profile ($a_{i}^{o},a_{-i}^{o}$) iff $\mathcal{S}(U_{i}(a_{i}^{o},a_{-i}^{o}),\gamma_{i}) \geqslant \mathcal{S}(U_{i}(a_{i}^{'},a_{-i}^{o}),\gamma_{i})$ $ \forall a_{i}^{'} \neq a_{i}^{o}$
\end{defi}
The above definition follows from the definition \ref{chrv:defi:rationalisable_basic} with an explicit reference to the condition of optimality of the equilibrium solution, that is, for the safety aspiration level of the agent $i$ to be rationalisable, their observed action $a_{i}^{o}$ must be the best response to the action of the other agents $a_{-i}^{o}$. Algo. \ref{algo:gamma_est_strategic} presents the general algorithm to estimate the rationalisable parameter. The intuition behind the algorithm is as follows: the value of the parameter $\gamma_{i}$ lies within the utility interval [-1,1]. Let $P=\{I_{1},I_{2},...,I_{P}\}$ be an ordered partition of the interval [-1,1]; the process of constructing the partition depends on the underlying models of reasoning and is explained later. We sample a single value of $\gamma \in I$ and check if the scalarization $\mathcal{S}$ based on that sampled value is rationalisable with respect to the definition \ref{chrv:def:1}. If so, we include the partition $I$ from which $\gamma$ was sampled in the set of rationalisable parameter set $\Gamma$, and the union of these sets is the set of rationalisable $\gamma$. Next, we set the condition under which the algorithm will be sound and complete.\par
\begin{proposition}
\label{prop:1}
Algorithm \ref{algo:gamma_est_strategic} is sound and complete based on a partition $P$ iff $\forall I \in P$,  \texttt{is\_rationalisable}($\gamma$) $\leftrightarrow$ \texttt{is\_rationalisable}($\gamma'$) $\forall \gamma, \gamma' \in I$
\end{proposition}
Implementing the \emph{is\_rationalisable} method based on the definition \ref{chrv:def:1} ensures soundness; this is because the utility maximizing action (condition of definition \ref{chrv:def:1}) in response to an equilibrium action means that the said action is in equilibrium, and therefore (correctly) rationalisable. The bidirectional implication condition of proposition \ref{prop:1} ensures that if we sample only a single value $\gamma$ from an interval $I \in P$ and check for rationalisability, then any $\gamma'$ in that interval that was not sampled is also rationalisable. Next, we construct a partition for which the double-implication condition of proposition \ref{prop:1} holds.
Given a game, let the partition $P_{\text{eq}}$ consist of the ordered safety utility of the agent $i$'s action as follows 
\begin{equation*}
P_{\text{eq}} = \{[-1, u_{i,s}(a_{i,1},a_{-i}^{o})),[ u_{i,s}(a_{i,1},a_{-i}^{o}), ...), [u_{i,s}(a_{i,|A_{i}|},a_{-i}^{o}), 1]\}
\end{equation*}
where $u_{i,s}(a_{i,1},a_{-i}^{o}) \leqslant ... \leqslant u_{i,s}(a_{i,|A_{i}|},a_{-i}^{o})$ is the ordered sequence of the safety utility values of agent $i$. An example partition for the game with respect to the row player (right turning vehicle) in response to the action $U$ (the observed action) of the column player of Fig. \ref{fig:weighted_strategic_example} is shown in Fig. \ref{fig:partition_line}.\\
\begin{figure}[h]
\centering
\includegraphics[width=.75\columnwidth]{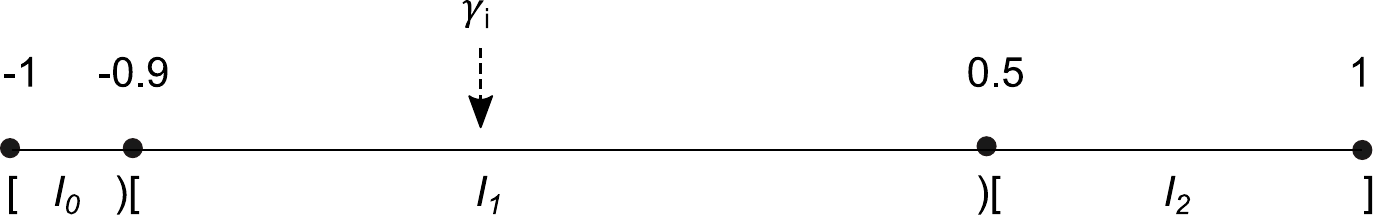}
  \caption{The partition intervals $I_{0},I_{1},..$ of $P$ based on the game of Fig. \ref{fig:weighted_strategic_example}.}
 \label{fig:partition_line}
\end{figure}
\begin{theo}
\label{chrp:theo_eq_rat}
For any interval $I \in P_{\text{eq}}$, if $\gamma \in I$ is equilibrium rationalisable, then $\forall \gamma' \in I$, $\gamma'$ is equilibrium rationalisable. Conversely, if $\gamma \in I$ is not equilibrium rationalisable, then $\forall \gamma' \in I$, $\gamma'$ is not equilibrium rationalisable.
\end{theo}
The proof is based on the intuition that if the partitions are constructed using ordered safety values of different actions, then any given threshold that falls between two such utilities imposes the same ordering of actions after scalarization since the conditions of Eqn. \ref{eqn:satis_form} remain unchanged.
\begin{proof}
\normalfont
\noindent Since the equilibrium rationalisability is based on the condition $\mathcal{S}(U_{i}(a_{i}^{o},a_{-i}^{o}),\gamma_{i}) \geqslant \mathcal{S}(U_{i}(a_{i}^{'},a_{-i}^{o}),\gamma_{i})$, we first show that $\mathcal{S}(U_{i}(a_{i},a_{-i}),\gamma_{i}) = \mathcal{S}(U_{i}(a_{i},a_{-i}),\gamma_{i}')$ $\forall \gamma, \gamma' \in I$, i.e., the scalarized value based on any two parameters that fall in the same interval is equal.\\
\noindent \textbf{Case} $u_{s,i}(a_{i},a_{-i}^{o}) \leqslant \min I$: In this case, $u_{s,i}(a_{i},a_{-i}^{o}) \leqslant \gamma$, $\forall \gamma \in I$. Therefore, $\mathcal{S}(U_{i}(a_{i},a_{-i}),\gamma_{i}) = \mathcal{S}(U_{i}(a_{i},a_{-i}),\gamma_{i}')$ since both evaluate to $u_{s,i}(a_{i},a_{-i}^{o})$ based on Eqn. \ref{eqn:satis_form}.

\noindent \textbf{Case} $u_{s,i}(a_{i},a_{-i}^{o}) > \min I$: In this case, $u_{s,i}(a_{i},a_{-i}^{o}) \geqslant \sup I$, since for any $u_{s,i}(a_{i},a_{-i}^{o}) \neq 1$, $u_{s,i}(a_{i},a_{-i}^{o}) = \min I$ when $u_{s,i}(a_{i},a_{-i}^{o}) \in I$ based on the construction of $P_{\text{eq}}$. Therefore, $\forall \gamma \in I$, $\gamma < u_{s,i}(a_{i},a_{-i}^{o})$, and $\mathcal{S}(U_{i}(a_{i},a_{-i}),\gamma_{i}) = \mathcal{S}(U_{i}(a_{i},a_{-i}),\gamma_{i}')$ since both evaluates to $u_{p,i}(a_{i},a_{-i}^{o})$ based on Eqn. \ref{eqn:satis_form}\\
\noindent Therefore for any $I \in P_{\text{eq}}$, the condition $\mathcal{S}(U_{i}(a_{i},a_{-i}),\gamma_{i}) = \mathcal{S}(U_{i}(a_{i},a_{-i}),\gamma_{i}')$ holds true for all $\gamma, \gamma' \in I$.\\
\noindent By the above equality condition, $\mathcal{S}(U_{i}(a_{i}^{o},a_{-i}^{o}),\gamma_{i}) \geqslant \mathcal{S}(U_{i}(a_{i}^{'},a_{-i}^{o}),\gamma_{i})$ $\leftrightarrow$ $\mathcal{S}(U_{i}(a_{i}^{o},a_{-i}^{o}),\gamma_{i}') \geqslant \mathcal{S}(U_{i}(a_{i}^{'},a_{-i}^{o}),\gamma_{i}')$, which establishes the biconditional relationship of the theorem based on the definition of equilibrium rationalisability (Defn. \ref{chrv:def:1}).
\end{proof}
Theorem \ref{chrp:theo_eq_rat} helps significantly reduce the number of consistency checks that we need to perform, since we need to check only one value in each interval in $P$ to determine whether all the values in that interval are rationalisable or not. This keeps the run-time complexity of Algo. \ref{algo:gamma_est_strategic} linear in the number of actions of the agent in the worst case (that is, $O(|A_{i}|)$), since the run time depends on the size of the partition $P_{\text{eq}}$, which in turn depends on the number of unique safety utilities, i.e., $|A_{i}|$ in the worst case.
\subsubsection{Non-strategic models}
Recall that for non-strategic models, an agent $i$ does not hold a specific belief about the action another agent might play, and therefore, similar to the weighted aggregation case, we cannot pin down a specific action $a_{-i}^{o}$ in response to which the parameters can be estimated. This leads to a revision of the rationalisability definition of Def. \ref{chrv:def:1} to make it independent of the actions of other agents for the \texttt{maxmax}
and \texttt{maxmin} models.
\begin{defi}
\label{chrv:def:2}
For any agent $i$, a safety aspiration level $\gamma_{i} \in [-1,1]$ is \texttt{maxmax} rationalisable with action $a_{i}^{o}$ iff $\max\limits_{a_{-i}} \mathcal{S}(U_{i}(a_{i}^{o},a_{-i}),\gamma_{i}) \geqslant \max\limits_{a_{-i}} \mathcal{S}(U_{i}(a_{i}^{'},a_{-i}),\gamma_{i})  \forall a_{i}^{'} \neq a_{i}^{o}$.\\
\end{defi}
\begin{defi}
\label{chrv:def:3}
For any agent $i$, a safety aspiration level $\gamma_{i} \in [-1,1]$ is maxmin rationalisable with action $a_{i}^{o}$ iff $\min\limits_{a_{-i}} \mathcal{S}(U_{i}(a_{i}^{o},a_{-i}),\gamma_{i}) \geqslant \min\limits_{a_{-i}} \mathcal{S}(U_{i}(a_{i}^{'},a_{-i}),\gamma_{i})  \forall a_{i}^{'} \neq a_{i}^{o}$.
\end{defi}
For the strategic case, we needed to check $a_{i}^{o}$ for rationalisability only as a response to a fixed action $a_{-i}^{o}$, and therefore it was sufficient to construct the partition $P_{\text{eq}}$ based only on the safety utilities for all actions that were in response to $a_{-i}^{o}$. However, for non-strategic models, the rationalisability of $a_{i}^{o}$ involves comparison with all entries of the safety utilities of agent $i$ in the game matrix. Therefore, to apply Prop. \ref{prop:1} for the non-strategic case, the partition points of $P$ need to include all the entries of the table as follows:
\begin{align*}
P_{\text{ns}} =& \{[-1, u_{i,s}(a_{i,1},a_{-i})),[ u_{i,s}(a_{i,1},a_{-i}),u_{i,s}(a_{i,2},a_{-i})),\\& ..., [u_{i,s}(a_{i,|A_{i}|},a_{-i}), 1]\}   
\end{align*}
where $u_{i,s}(a_{i,1},a_{-i}) \leqslant u_{i,s}(a_{i,2},a_{-i}) \leqslant ... \leqslant u_{i,s}(a_{i,|A_{i}|},a_{-i})$ is the ordered sequence of the safety utility values of agent $i$, and (with a minor abuse of notation) $a_{-i}$ steps through all the corresponding actions of the other agents based on that ordering. The only difference between $P_{\text{eq}}$ and $P_{\text{ns}}$ is that $P_{\text{ns}}$ is partitioned based on the safety utilities of $i$ in the entire game matrix, whereas $P_{\text{eq}}$ was based on the column corresponding to $a_{-i}^{o}$. This also has an impact on the runtime of the algorithm, which is $O(|A|^{N})$, where $N$ is the number of players in the game and $|A|$ is the number of actions for a player. This value is the same as the size of the game matrix since the partition is constructed from each safety utility value for each agent. The corresponding corollaries of Theorem \ref{chrp:theo_eq_rat} for the non-strategic case are as follows:
\begin{cor}
\label{chrp:cor_eq_rat_maxmax}
For any interval $I \in P_{\text{ns}}$, if $\gamma \in I$ is maxmax rationalisable, then $\forall \gamma' \in I$, $\gamma'$ is maxmax rationalisable. Conversely, if $\gamma \in I$ is not maxmax rationalisable, then $\forall \gamma' \in I$, $\gamma'$ is not maxmax rationalisable.
\end{cor}
\begin{cor}
\label{chrp:cor_eq_rat_maxmin}
For any interval $I \in P_{\text{ns}}$, if $\gamma \in I$ is maxmin rationalisable, then $\forall \gamma' \in I$, $\gamma'$ is maxmin rationalisable. Conversely, if $\gamma \in I$ is not maxmin rationalisable, then $\forall \gamma' \in I$, $\gamma'$ is not maxmin rationalisable.
\end{cor}
\begin{proof}
The proof of the above corollaries is similar to the proof of Theorem \ref{chrp:theo_eq_rat}. Observe that for the partition set $P_{\text{ns}}$, for any agent $i$, the aggregation of the utilities of $i$ is the same for any pair of $\gamma, \gamma' \in I$. This follows from the equality condition $\mathcal{S}(U_{i}(a_{i},a_{-i}),\gamma_{i}) = \mathcal{S}(U_{i}(a_{i},a_{-i}),\gamma_{i}')$, which holds for the partition $P_{\text{ns}}$ in the same way as was established for $P_{\text{eq}}$ earlier in Theorem \ref{chrp:theo_eq_rat}. This means that the pairwise comparison between the utilities of actions of $i$ is invariant to the value of $\gamma \in I$, thus establishing the conditions of definitions \ref{chrv:def:2} and \ref{chrv:def:3}.
\end{proof}
\section{Experiments and evaluation}
\label{sec:experiments}
\FloatBarrier
We evaluate the utility aggregation estimation methods using a multi-agent traffic drone data set from \citep{sarkar2021solution}, which includes several real-world traffic interactions. From the dataset, we include interaction scenarios of unprotected right-turn and left-turn across path at a busy four way signalized intersection (1667 2-player games), decisions about entering a roundabout (2441 N-player games), and decisions about waiting for a pedestrian at a crosswalk (288 N-player games). Each of these games are constructed from the perspective of a \emph{focal} agent about to enter the intersection, roundabout, and crosswalk, respectively. The intersection and roundabout games contain the interaction of the focal agent (a vehicle) with other vehicles, and the crosswalk scenario include games in which the other players are pedestrians. The relevant code is available at \url{git.uwaterloo.ca/a9sarkar/single-shot-hierarchical-games}. As part of our analysis, first, we evaluate the proportion of games in which a rationalizable parameter was found for each aggregation and reasoning model combination. Second, we evaluate the improvement in accuracy of player's predicted actions when the utilities are constructed using the proposed estimation methods.
\begin{small}
\subsection{Agent Utility Aggregation Parameters}
\FloatBarrier
\begin{table*}[ht]
\begin{center}
\begin{tabular}{p{1.5cm}l|cccccc}
    \toprule
    &&\multicolumn{2}{c}{Intersection}&\multicolumn{2}{c}{Roundabout}&\multicolumn{2}{c}{Crosswalk}\\
    \cmidrule(lr{1em}){3-8}
    &&Weighted&Satisficing&Weighted&Satisficing&Weighted&Satisficing\\
    \midrule
    Strategic&Nash& 100\% & 68\% & 100\% & 96.9\% & 100\% & 100\% \\
    \midrule
    \multirow{2}{*}{Non strategic}&maxmax& 100\% & 72.2\% & 100\% & 56.8\% & 100\% & 66.9\% \\
    &maxmin& 100\% & 72.28\% & 100\% & 56.87\% & 100\% & 64.3\% \\
    \bottomrule
\end{tabular}
\end{center}
\caption{Pass rate of estimated preferences for each model, aggregation method, and dataset.}
\label{chrv:tab:pass_rate}
\end{table*}
\end{small}
The first point of analysis is the pass rate for each model, that is, the percentage of games in which a rationalisable parameter was found for each model and aggregation method combination (Table \ref{chrv:tab:pass_rate}). Based on the estimation procedure, we see that the chosen action of drivers can be rationalised by a weighted aggregation parameter in all cases for all models. For satisficing aggregation, the pass rate is sensitive to the specific traffic situation and the choice of the reasoning model. In the roundabout and crosswalk scenarios, rationalisable parameters for strategic models could be estimated for almost all games (96.9\% and 100\%, respectively), whereas for non-strategic models it could only be found for 56.8\% to 66.9\% of the games depending on the specific solution concept. There are two possible reasons why weighted aggregation parameters show higher rationalisability. First, at least for non-strategic models, the optimisation method involves an approximate procedure (in the form of the use of trust region based method of \citep{curtis2010interior}), which ends up finding a solution, albeit approximate, more easily than the corresponding exact estimation procedure for satisficing based methods. Second, it might be possible that a weighted aggregation methodology as opposed to a satisficing based procedure makes the chosen action more optimal with respect to rationalisability constraints. The actual parameter values of the weights shed more light on this aspect, which is discussed next.\par
\begin{figure}[t]
\centering
\includegraphics[width=0.9\columnwidth]{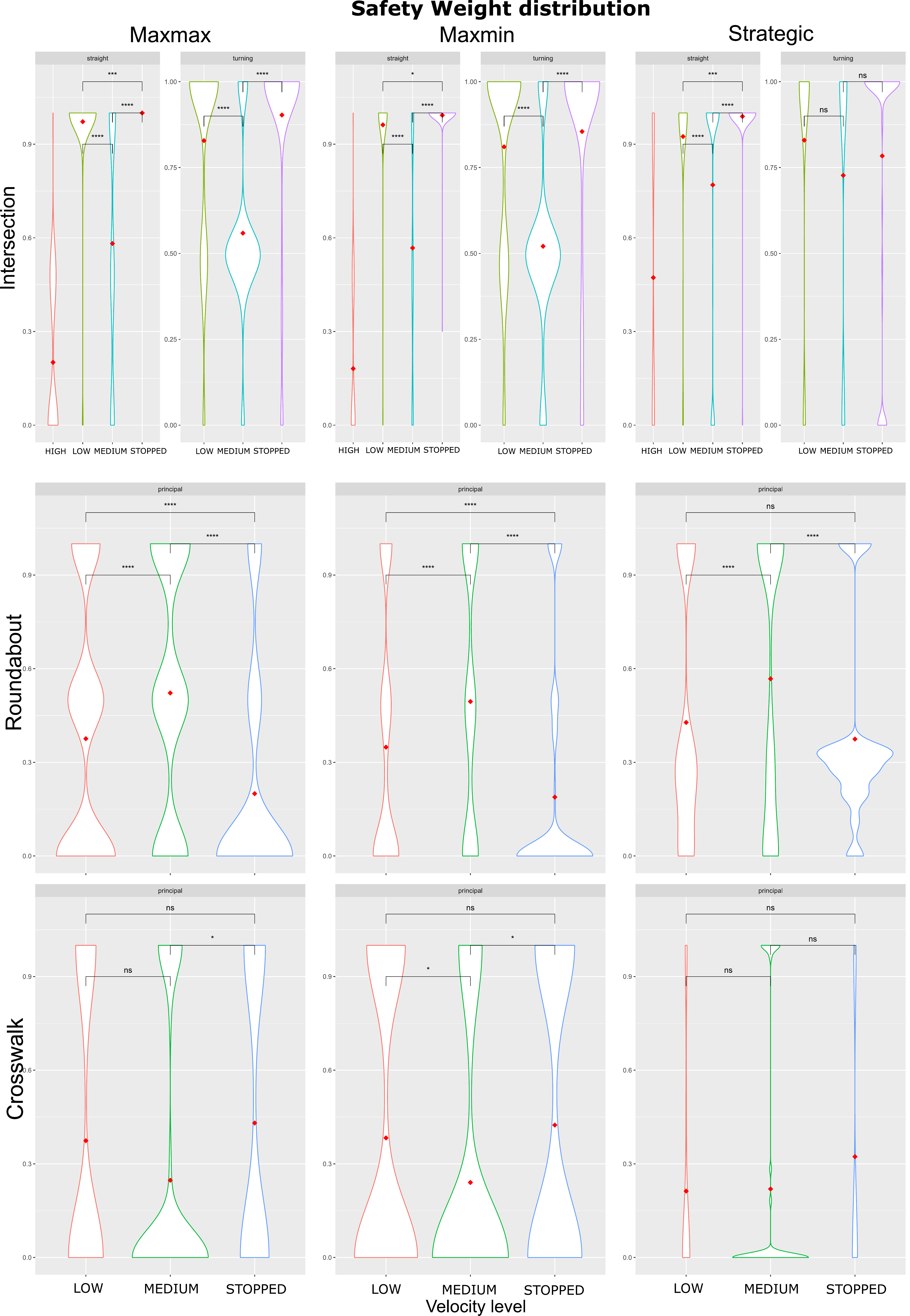}
  \caption{Safety weight parameter distribution stratified by vehicle speed, scenario, and task. Significance levels are noted as $p\leqslant0.05 (*), p \leqslant 0.01(**), p \leqslant 0.001(***), p \leqslant 0.0001(****)$, and \emph{ns}.}
 \label{fig:weighted_pref_disaggr_params}
\end{figure}
\begin{figure}[t]
\centering
\includegraphics[width=0.9\columnwidth]{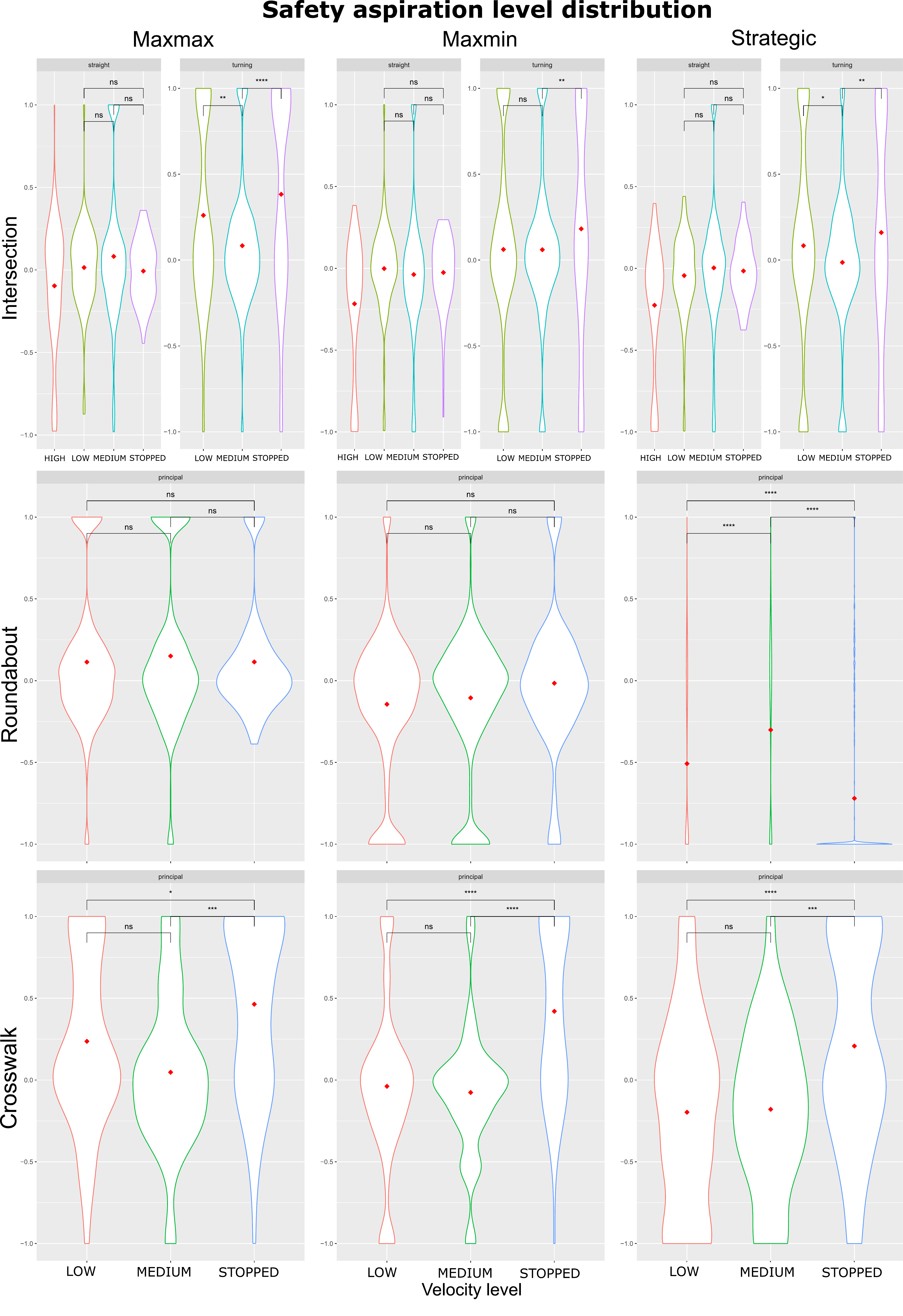}
  \caption{Safety aspiration level distribution stratified by vehicle speed, scenario, and task. Significance levels are noted as $p\leqslant0.05 (*), p \leqslant 0.01(**), p \leqslant 0.001(***), p \leqslant 0.0001(****)$, and \emph{ns}.}
 \label{fig:satisficing_pref_disaggr_params}
\end{figure}
\subsubsection{Parameter values} Next, we study the values of the rationalisable parameters that were estimated under each model. Figures \ref{fig:weighted_pref_disaggr_params} and \ref{fig:satisficing_pref_disaggr_params} show the violin plots of the safety weight (safety weight of $\mathbf{w}$) and safety aspiration level ($\gamma$) for weighted and satisficing aggregation methods, respectively, for the focal player in roundabout and crosswalk games, and both players in intersection games. For cases where rationalisable parameters are an interval, for the purpose of the figure we select the center of the interval as the representative sample. To study the association between the state factors, the figures are stratified based on the velocity of the agent as well as the scenario. The first observation about the weighted aggregation method is that irrespective of the reasoning model, the distributions of the weight parameters are multimodal with the modes being concentrated towards lower and higher values in most cases. This means that in most cases, with safety and progress as the two objectives, drivers tend to weigh heavily one or the other rather than weighing both together in some mixed proportion at the same time. Intuitively, this makes sense because what the revealed preference estimation \emph{rationalises} are the weights of the two objectives that would make the chosen action of the driver optimal. For example, in a given game, if the driver chose to proceed, then what we find is that evaluating that action only with respect to its, say, progress utility, makes it more optimal than if the driver had evaluated that action based on both safety and progress utility. On the other hand, for the satisficing method of aggregation, the distribution of $\gamma$ is not multimodal in most cases. Rather, the mean values (shown in red) are concentrated near 0, thereby indicating more stability across reasoning model or the game situation.\par

\subsubsection{Subgroup analysis} Next, we study the association between the vehicle speed and the estimated parameters values. The median values of the parameters are shown in red within the violin plots. Within each dataset and scenario, we perform subgroup analysis based on discretised velocity levels\footnote{Some velocity levels were excluded due to not enough instances in the data}, and significance between groups is noted above the horizontal lines according to Wilcoxon \emph{t-test} at significance levels $p\leqslant0.05 (*), p \leqslant 0.01(**), p \leqslant 0.001(***), p \leqslant 0.0001(****)$ (null hypothesis of equal means between groups). In general, significant differences in parameter values with respect to velocity levels were found for 72\% and 42\% of pairwise group comparisons for weighted and satisficing aggregation, respectively, where a group is a combination of scenario and model. This points to the fact that the safety aspiration levels of drivers show more stability at different velocity levels compared to the weight parameters. Additionally, within each scenario and reasoning model, for the cases where there is a significant difference, higher velocities are associated with lower safety weights and lower safety aspiration level for weighted and satisficing aggregation, respectively.

\FloatBarrier
\subsection{Predictive Accuracy Improvements}
Since subgroup analysis shows an association is observed between factors such as speed and task with the values of the aggregation parameter, in this section, we evaluate whether the aggregation methods in conjunction with a statistical learning method, such as Classification and Regression Trees (CART) \citep{breiman2017classification}, can improve the predictive accuracy of various behaviour models and solution concepts with respect to naturalistic human behaviour. We first split the dataset of games into training and testing set in 80:20 proportion with $K=30$ random subsampling, and use the training set to construct two separate regression trees, one each for weighted and satisficing aggregation, respectively. The independent variables in the models are driving task (left-turn, right-turn, straight-through, etc.), scenario (intersection, crosswalk, roundabout), reasoning model, and velocity, whereas the aggregation parameter estimated using the methods developed earlier is the dependent variable. We use the CART model to predict the aggregation parameters in the games in the testing set, and solve the games using different solution concepts. We then evaluate the improvement in accuracy of prediction of each solution concept when the games are constructed with multi-objective utilities aggregated using a fixed weights of $\mathbf{w}=[0.5,0.5]$ compared to our method. We select solution concepts from existing literature that have been proposed as a model for human driving or autonomous vehicle behaviour. 
The selected solution concepts are as follows: Level-k model \citep{tian2018adaptive, sarkar2021solution}, where $k=\{0,1,2\}$ indexed as L0, L1, and L2, respectively. L0 models are further separated based on the non-strategic solution concept into \texttt{maxmax} (L0:MX) and \texttt{maxmin} models (L0:MM) \citep{wright2014level}, pure strategy Nash equilibrium (PNE) with welfare maximizing solution selected in cases of multiple equilibria  \citep{sarkar2021solution,schwarting2018planning,pruekprasert2019decision}, Stackelberg solution for 2-player scenarios (Stack.) \citep{sun2020game, fisac2019hierarchical}, Rule based solution that says always wait for other vehicles or pedestrians, and a level-k, $k=1$,  model in which the level-0 behaviour is the rule based behaviour (LkR).\par
In the Stackelberg model, the focal agent, on account of not holding the right of way, is modelled as the follower. To match the observed (ground truth) manoeuvre with one of the two manoeuvres in our games (\emph{wait} or \emph{proceed}), we first select the trajectory generated in the game that is closest to the observed trajectory based on the trajectory length. The manoeuvre corresponding to that closest trajectory is selected to be the ground-truth manoeuvre.\par
\begin{figure}[t]
\centering
\includegraphics[width=\columnwidth]{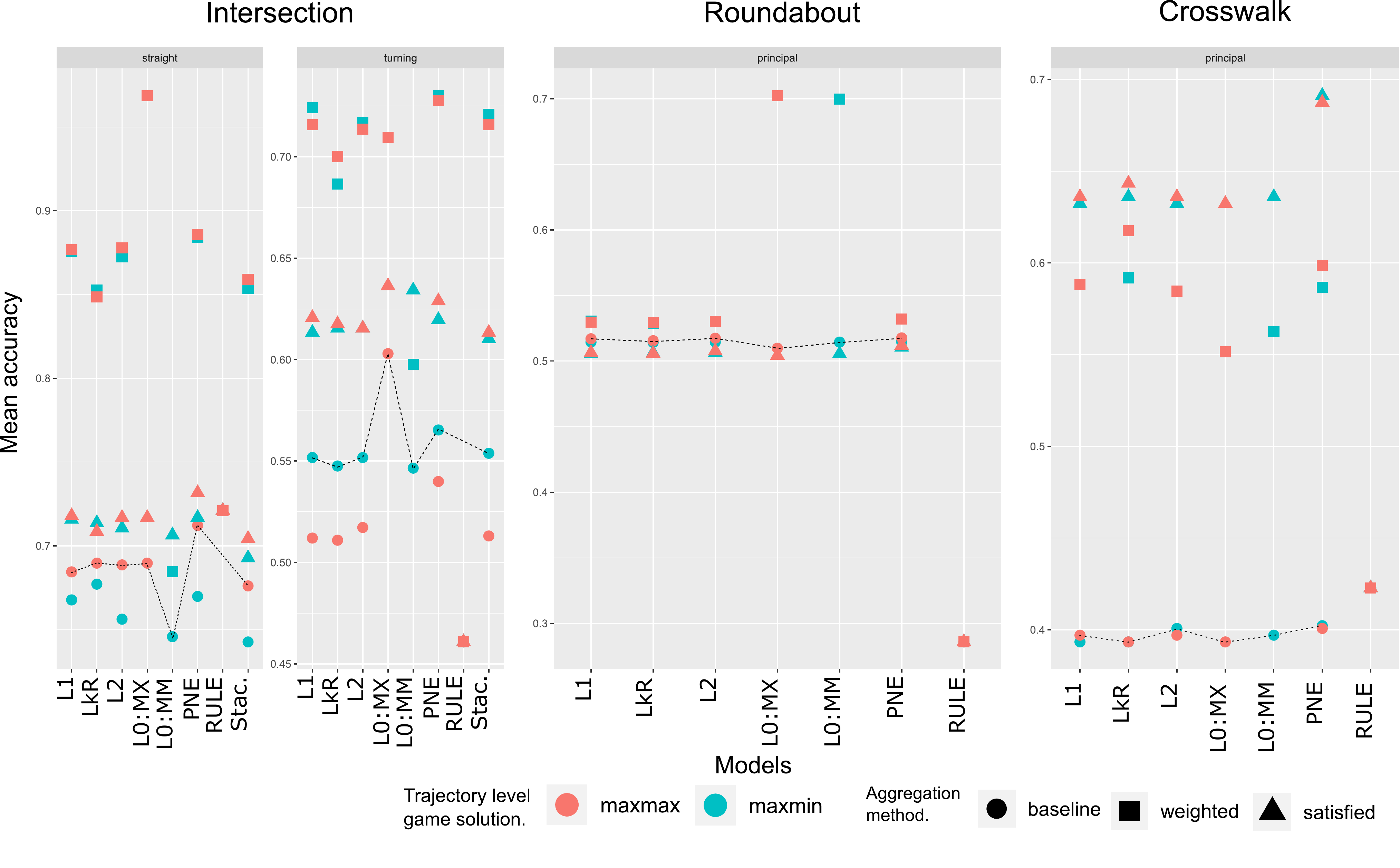}
  \caption{Mean accuracy of all the models. The dashed line highlights the accuracy of the models constructed with baseline weights.}
 \label{fig:mean_accuracy_all}
\end{figure}

Fig. \ref{fig:mean_accuracy_all} shows the mean accuracy, i.e., proportion of match between predicted and observed manoeuvre, across 30 runs. Compared to a baseline (weighted aggregation $\mathbf{w}=[0.5,0.5]$ indicated as dashed line in figure Fig. \ref{fig:mean_accuracy_all}), a prediction-based aggregation consistently shows better accuracy regardless of the choice of the solution concept. The performance of the models with respect to weighted and satisficing aggregation show some dependency on the specific scenario. For intersection and roundabout scenarios, weighted aggregation shows higher accuracy; whereas for crosswalk, satisficing based aggregation shows higher accuracy. The crosswalk scenario, which is vehicle-pedestrian interaction type, is quite different compared to the intersection or roundabout where there are only vehicle-vehicle interaction games. Drivers are also much more cautious when navigating a crosswalk, since there are pedestrians involved. Combining the insights from the accuracy result along with the pass rate of Table \ref{chrv:tab:pass_rate}, the data suggest that satisficing is much more effective as an aggregation method in scenarios where drivers exhibit higher levels of caution, such as crosswalk navigation. The final observation is the worse performance of pure rule following indicating that a model of pure rule following might not be best suited as a model of behaviour for the selected situations of high strategic interactions.\par

\FloatBarrier
\section{Conclusion}
In this paper, we address the problem of estimation of multi-objective aggregation parameters of agents based on observed behaviour. The methods developed in the paper are based on the idea of \emph{rationalisability}, i.e., a value of the aggregation parameter that makes the observed decision of a player optimal conditioned upon a reasoning model. The paper covers two processes of aggregation, namely weighted and satisficing aggregation, and the reasoning models cover strategic as well as non-strategic models. We show that the process of estimating aggregation parameters for weighted aggregation can be formulated as a linear and non-linear program for strategic and non-strategic models, respectively. Furthermore, we develop a novel algorithm for estimation of aggregation parameters for satisficing aggregation that is linear time for strategic models and polynomial time in the size of action space for non-strategic models. Based on a naturalistic dataset of three different traffic scenarios, rationalisable parameters for weighted aggregation were found for all games in the dataset, and for the majority of the games for satisficing aggregation. The paper also includes an extensive evaluation of several solution concepts and behaviour models and shows that, compared to a weighted aggregation with fixed set of weights, the proposed method of utility aggregation improves predictive accuracy across all the chosen solution concepts. In future, the hypotheses generated by the methods in the paper using observational data can be followed up with a controlled experiment to evaluate the aggregation process of human drivers in different situations.   
\balance

\bibliographystyle{ACM-Reference-Format} 
\bibliography{sample}

%%%%%%%%%%%%%%%%%%%%%%%%%%%%%%%%%%%%%%%%%%%%%%%%%%%%%%%%%%%%%%%%%%%%%%%%

\end{document}